\documentclass[a4paper,fleqn,longmktitle]{cas-sc}

\usepackage[numbers]{natbib}
\pdfoutput=1
\usepackage{subfigure}
\usepackage{graphicx}
\usepackage{caption}
\usepackage{booktabs} 
\newtheorem{definition}{Definition}[section]
\newtheorem{proposition}{Proposition}[section]
\newtheorem{proof}{Proof}[section]
\def\tsc#1{\csdef{#1}{\textsc{\lowercase{#1}}\xspace}}
\tsc{WGM}
\tsc{QE}
\tsc{EP}
\tsc{PMS}
\tsc{BEC}
\tsc{DE}

\begin{document}

\def\textpagefraction{.001}
\shorttitle{}
\shortauthors{Bingxue Wu et~al.}

\title [mode = title]{ Entropy Regularized Iterative Weighted Shrinkage-Thresholding Algorithm (ERIWSTA): An Application to CT Image Restoration}


\author[mymainaddress]{Bingxue Wu}

\author[mymainaddress]{Jiao Wei}
\author[mymainaddress]{Chen Li}
\author[mythirdaryaddress]{Yudong Yao}

\author[mymainaddress, mysecondaryaddress]{Yueyang Teng\corref{mycorrespondingauthor}}
\cortext[mycorrespondingauthor]{Corresponding author}
\ead{tengyy@bmie.neu.edu.cn}

\address[mymainaddress]{College of Medicine and Biological Information Engineering, Northeastern University, Shenyang  110169, China}
\address[mysecondaryaddress]{Key Laboratory of Intelligent Computing in Medical Image, Ministry of Education, Shenyang  110169, China}
\address[mythirdaryaddress]{Department of Electrical and Computer Engineering, Stevens Institute of Technology, Hoboken, NJ 07102, USA}

\begin{abstract}
The iterative weighted shrinkage-thresholding algorithm (IWSTA) has shown superiority to the classic unweighted iterative shrinkage-thresholding algorithm (ISTA) for solving linear inverse problems, which address the attributes differently. This paper proposes a new entropy regularized IWSTA (ERIWSTA) that adds an entropy regularizer to the cost function to measure the uncertainty of the weights to stimulate attributes to participate in problem solving. Then, the weights are solved with a Lagrange multiplier method to obtain a simple iterative update. The weights can be explained as the probability of the contribution of an attribute to the problem solution. Experimental results on CT image restoration show that the proposed method has better performance in terms of convergence speed and restoration accuracy than the existing methods.
\end{abstract}


\maketitle

\section{Introduction}
For sparse compressed signals, Donoho $et~al.$ \cite{2005Stable} proposed a compressive sensing theory that enables efficient data sampling at a much lower rate than the requirements, which can be modeled as follows in its standard formulation.

$\bf{Notations}$: In this paper, the matrices are represented in capital letters. For a matrix $A$, $A_{*i}$, $A_{i*}$ and $A_{ij}$ denote the $i-th$ column, the $i-th$ row and $(i,j)-th$ element of $A$, respectively; the $\|\dot\|_i$ represents the $i$-norm of a vector. All the vectors are column vectors unless transposed
to a row vector by a prime superscript $T$.

Compressive sensing can be formulated as:
\begin{eqnarray}
	b = Ax+\epsilon
\end{eqnarray}
where $x\in R^n$ is an unknown vector, $b\in R^m$ is an observed vector, $A\in R^{m\times n}$ is called the compressive sensing matrix (usually $m<<n$), and $\epsilon$ is the unknown disturbance term or noise. Obviously, this is an underdetermined system of equations that does not have a sole solution. The least-squares method is usually used to solve the problem.
\begin{equation}
	\min \limits_x  \frac{1}{2}\|Ax - b\|_2^2 
\end{equation}

To suppress overfitting, some scholars \cite{2010On,2010Sparse,1999Sparse,2017A} added the $L_0$-norm regularizer to introduce sparse prior information.
\begin{equation}\label{eq:L0}
	\min\limits_x \frac{1}{2}\| Ax - b\|_2^2 + \beta\|x\|_0
\end{equation}
where $\|x\|_0$ denotes the number of nonzero components of $x$ and $\beta>0$ is a hyperparameter to control the tradeoff between accuracy and sparsity. Many methods have been developed to solve this problem, such as the penalty decomposition method \cite{2012Sp}, iterative hard threshold method \cite{2008Iterative}, fixed-point continuation method (FPC) \cite{2008Fixed}, approximate gradient homotopy method (PGH) \cite{2012A} and reweighted $L_0$ minimization method \cite{6743943,2015Zhao}.

However, Eq. (\ref{eq:L0}) is an NP-hard optimization problem \cite{1995Sparse}, which is highly discrete so that it is difficult to solve using a precise algorithm. Thus, we need to seek an effective approximation solution for this problem. The $L_1$-norm regularizer is introduced as a substitute for that of the $L_0$-norm. Such an approximation can be traced back to a wide range of fields, such as seismic traces \cite{1979Deconvolution}, sparse signal recovery \cite{2001Atomic}, sparse model selection in statistics (LASSO) \cite{1996Regression}, and image processing \cite{1970Total}. Many scholars have attempted to find the optimal solution to the following problem:
\begin{equation}
	\min\limits_x \frac{1}{2}\|Ax - b\|_2^2 + \beta \| x \|_1
\end{equation}
It is a convex continuous optimization problem with a sole nondifferentiable point ($x=0$), which can usually be transformed into a second-order cone programming problem and then solved by methods such as interior-point methods. However, in large-scale problems, due to the high algorithmic complexity, the interior-point method is very time-consuming. Based on this, many researchers have solved the problem through simple gradient-based methods. Among them, the iterative shrinkage-thresholding algorithm (ISTA) proposed by Chambolle $et~al.$ \cite{Chambolle1998,2003An} has attracted much attention. ISTA updates $x$ through a shrinkage/soft threshold operation in each iteration.
\begin{equation}
	x^{k + 1} = soft_{\beta t}[x^k - 2tA^T( {Ax^k - b} )]
\end{equation}
where $k$ represents the $k$-th iteration, $t$ is an appropriate stepsize and $soft$ is the soft threshold operation function.
\begin{equation} 
	soft_\theta (x_i) = sign(x_i)( \|x_i\| - \theta )
\end{equation}

Recently, the iteratively weighted shrinkage-thresholding algorithm (IWSTA) has attracted much interest compared with ISTA, which outperforms their unweighted counterparts in most cases. In these methods, decision variables and weights are optimized alternatingly, or decision variables are optimized under heuristically chosen weights. It can be written as:
\begin{equation}
	\min \limits_{x,~w\geq0} \frac{1}{2}\| Ax - b\|_2^2 + \beta \sum\limits_{i = 1}^n w_i |x_i|_1 
\end{equation}

The method assigns different weights to each component of $x$ in the iterative process and then updates $x$. In this way, each subproblem is convex and easy to solve. Many algorithms have been developed to solve it. For example, the iterative support detection (ISD) method \cite{2009Wang} assigns a weight of 0 to components in the support set and a weight of 1 to the other components during iteration, in which the support set at each iteration consists of all components whose absolute value is greater than the threshold. Zhao $et~al.$ \cite{2015Zhao} proposed a new method to calculate the optimal $w$ by the duality of linear programming based on the property of weighted range space. It alternately solves the weighted original problem with fixed weights to obtain a new solution $x$, and then it solves the duality problem to obtain a new weight $w$. More variants are available in \cite{6743943,David2006For} and its references. The details of some examples are listed in Tab. \ref{tab:existmethods}.

\begin{table}[width=.9\linewidth,cols=4,pos=h]\label{tab:exist}
	\caption{Variants of weighted method.}
	\label{tab:existmethods}
	\begin{tabular*}{\tblwidth}{@{} LCCCCC@{} }
		\toprule
		Author & Termed &  Weights &Min.&Max. &  Regularizer\\
		\midrule		
		Chambolle $et~al.$ \cite{Chambolle1998} & ISTA & 1 & 1&1 &$\sum\limits_{i = 1}^n | x_i| $ \\
		
		Candes $et~al.$ \cite{EJ2008} & IRL1  & 
		$\frac{1}{|x_i^{k - 1}| + \delta }$& 0&$\frac{1}{\delta}$&  $\sum\limits_{i = 1}^n \log (| x_i| + \delta) $ \\
		
		Foucart $et~al.$ \cite{2009Sparsest}  & WLP &$\frac{1}{( {|x_i^{k - 1}| + \delta } )^{1 - p}}$& 0&$\frac{1}{\delta^{1-p}}$& $\frac{1}{p}\sum\limits_{i = 1}^n ( |x_i| + \delta ) ^p$ \\	
		
		Wipf $et~al.$ \cite{2010Iterative} & NW4 & $\frac{1+(|x^{k - 1}| + \delta)^{p + 1}}{( | x^{k - 1}| + \delta )}^{p + 1}$ & 0&1&$\sum\limits_{i = 1}^n ( |x_i | - \frac{1}{(x_i+ \delta)^p})$ \\
		\bottomrule
	\end{tabular*}
\end{table}

There is a drawback for the above methods: the weights do not meet the usual definition of weights, and their sum is one, which leads them to be distributed in a very large range (see Tab. \ref{tab:existmethods}). Such weights are difficult to explain and can lead to an inaccurate result.

This paper proposes a new IWSTA type, called entropy regularized IWSTA (ERIWSTA), which obtains easily computable and interpretable weights. The weights automatically fall in the range of [0, 1], and the summation is
one so that they can be considered a probability of the contribution of each attribute to the model.
This is achieved by adding an entropy regularizer
to the cost function and then using the Lagrange multiplier
method to solve the problem. Experiments are executed for CT image restoration, and the results show that the proposed algorithm performs better in terms of both convergence speed and restoration accuracy compared with some state-of-the-art methods.

\section{Methodology}
The main idea of the IWSTA type algorithms is to define a weight for each attribute based on the current iteration ${x^k}$ and then use them to obtain a new $x$. In this section, we introduce an entropy regularizer to the cost function and obtain the following optimization model:
\begin{eqnarray}\label{eq:mainmodel}
	\min &&\Phi _{\beta ,\gamma }(x,w) = F(x)+ \beta G_{\gamma}(x,w)\nonumber\\
	s.~t. && w_i\geq 0, ~\sum_{i=1}^n w_i =1\nonumber\\
	where&&F(X) = \frac{1}{2}\|Ax-b\|_2^2\nonumber\\
	&&G_{\gamma}(x,w)=\sum_{i=1}^n w_i |x_i|+\gamma\sum_{i=1}^n  {w_i}\ln {w_i} 
\end{eqnarray}
where
$\gamma\geq0$ is a given hyperparameter.

As can be seen, while we do not use the entropy regularizer, $w$ can easily be solved as $w_i=1$ if $x_i=\arg\min\{|x_1|, ..., |x_n|\}$, or 0 otherwise$\footnote{The update rule can be easily explained by an example as \begin{equation}
		\begin{aligned}
			\min~\{4, -1, 5\}= min&~4w_1-1w_2+5w_3\\ 
			s.~t.&~ w_1, w_2, w_3\ge0\\
			&~w_1+w_2+w_3=1\nonumber
			\label{eq9}
		\end{aligned}
	\end{equation}
	The solution is $w_1=0$, $w_2=1$ and $w_3=0$, in which $w_2$ corresponds to the minimum value of \{4, -1, 5\}. It is very similar to the computation of the weights in the k-means algorithm. }$. It shows a simple fact that only one element of $w$ is 1, and the others are 0, which is grossly incompatible with the actual problem. Then, we add the negative entropy of the weights to measure the uncertainty of
weights and stimulate more
attributes to help signal reconstruction because it is well known that $\sum_{i=1}^n  {w_i}\ln {w_i}$ is minimized in information theory when
\begin{equation}
	w_1=w_2=...=w_n
\end{equation}
As follows, we will alternatively solve $w$ and $x$
in Eq. (\ref{eq:mainmodel}).

\subsection{Update rule for $w$}
To solve $w$, we introduce the Lagrange
multipliers $\lambda$ and then obtain the following Lagrange function.
Note that $F(x)$ is a constant with respect to $W$, so we only construct the Lagrange function on $G(x)$.
\begin{equation}
	L(w,\lambda)  =  G_{\gamma}(x,w) +  \lambda(\sum_{i = 1}^n w_i - 1),
\end{equation}
Set the partial derivative of $L(w,\lambda) $ with respect to $w_i$ and $\lambda$ to zero and then obtain the following two equations.
\begin{eqnarray}
	\frac{\partial L(w,\lambda)}{\partial w_i}&=& |x_i| + \gamma (1 + \ln {w_i})+\lambda = 0\label{eq:Lag1} \\
	\frac{\partial L(w,\lambda)}{\partial \lambda}&=&	\sum_{i = 1}^n w_i - 1=0\label{eq:Lag2}
\end{eqnarray}
From Eq. (\ref{eq:Lag1}), we know that
\begin{equation}\label{eq:wi}
	w_i =  \exp(- \frac{\lambda}{\gamma})\exp(-\frac{|x_i|}{\gamma}) 
\end{equation}
Substituting Eq. (\ref{eq:wi}) into Eq. (\ref{eq:Lag2}), we have
\begin{equation}
	\sum_{i = 1}^n {w_i}  = \exp(- \frac{\lambda}{\gamma})\sum_{i = 1}^n \exp(-\frac{|x_i|}{\gamma})  = 1
\end{equation}
It follows that
\begin{equation}
	\exp(- \frac{\lambda}{\gamma})=\frac{1}{\sum_{i = 1}^n \exp(-\frac{|x_i|}{\gamma})} 
\end{equation}
Substituting this expression to Eq. (\ref{eq:wi}), we obtain that

\begin{equation} 
	w_i = \frac{\exp(-\frac{|x_i|}{\gamma})}{\sum_{l = 1}^n \exp(-\frac{|x_l|}{\gamma})}
\end{equation}
Such weights certainly meet the constraints that $w_i\geq0$ and $\sum_{i = 1}^n w_i=1$.

\subsection{Update rule for $x$}
Inspired by the work of ISTA \cite{2009A}, a similar approach was adopted for the iterative update of $x$.
The construction of a  majorization is an important step in
obtaining the updating rule.
\begin{definition}\label{definition:surrogate}(Majorization)
	Denote $\psi(x|x^k)$ as a majorization for ${\Psi}(x)$ at $x^k$ (fixed)  if $\psi(x^k|x^k)={\Psi}(x^k)$ and
	${\psi}(x|x^k)\geq {\Psi}(x)$.
\end{definition}
Clearly, ${\Psi}(x)$ is nonincreasing under the updating rule $x^{k+1}=\min_x
\psi(x|x^k)$ because
\begin{eqnarray}
	{\Psi}(x^{k+1})\leq \psi(x^{k+1}|x^k)\leq \psi(x^k|x^k)={\Psi}(x^k)
\end{eqnarray}

Then, we can construct the  majorization for $F(x)$.
\begin{proposition}
	Obviously, $F(x)$ is a Lipschitz continuous and differentiable convex function, which has a  majorization function at fixed current iteration $x^k$ as
	
	\begin{eqnarray}
		f(x,x^k) = F(x^k)+[\nabla F(x^k)]^T(x-x^k)+\frac{L}{2}\|x-x^k\|_2^2
	\end{eqnarray}
	where $L$ is larger than or equal to the maximum eigenvalue of $A^TA$.
\end{proposition}
\begin{proof}
	It is well-known that
	\begin{equation}
		F(x) = \frac{1}{2}\|Ax-b\|_2^2=F(x^k)+[\nabla F(x^k)]^T(x-x^k)+\frac{1}{2}(x-x^k)^TA^TA(x-x^k)
	\end{equation}
	We compare $F(x)$ and $f(x,x^k)$ and find that only the last terms are different.
	By singular value decomposition (SVD) of a symmetric definite matrix, we know that $A^TA=Q^T\Sigma Q$, in which $Q$ is an orthogonal matrix consisting of all eigenvectors and $\Sigma$ is diagonal consisting of all eigenvalues. Let $z=x-x^k$, then
	\begin{equation}
		z^T(A^TA)z=z^TQ^T\Sigma Qz\leq L\|Qz\|_2^2=L\|z\|_2^2
	\end{equation}
	And it is also certain that $z^TA^TAz= L\|z\|_2^2=0$ if $x=x^k$. Thus, the proof is established.
\end{proof}
Now, we obtain the  majorization for the cost function $\Phi(x,w)$ on $x$.
\begin{equation}
	\phi(x,x^k)=f(x,x^k)+\beta G_{\gamma}(x,w) 
\end{equation}
which can be reorganized as
\begin{eqnarray}
	\phi(x,x^k)&=& \frac{L}{2}\| {x - [x^k - \frac{1}{L}\nabla F ( x^k )]} \|_2^2 + \beta G_{\gamma}(x,w)\nonumber\\
	&=&\sum_{i=1}^n\{\frac{L}{2}\| x_i - [x^k - \frac{1}{L}\nabla F ( {{x^k}}) ]_i \|_2^2 + \beta w_i|x_i|\}+constant
\end{eqnarray}
We find that the variables of the  majorization are separable such that their minimizations can be easily obtained on each $x_i$, respectively, as follows:
\begin{equation}
	x_i^{k + 1} = soft_{\beta  t w_i}[x^k - 2tA^T( {A{x^k} - y} )]
\end{equation}

\begin{figure*}
	\centering
	\includegraphics[scale=.4]{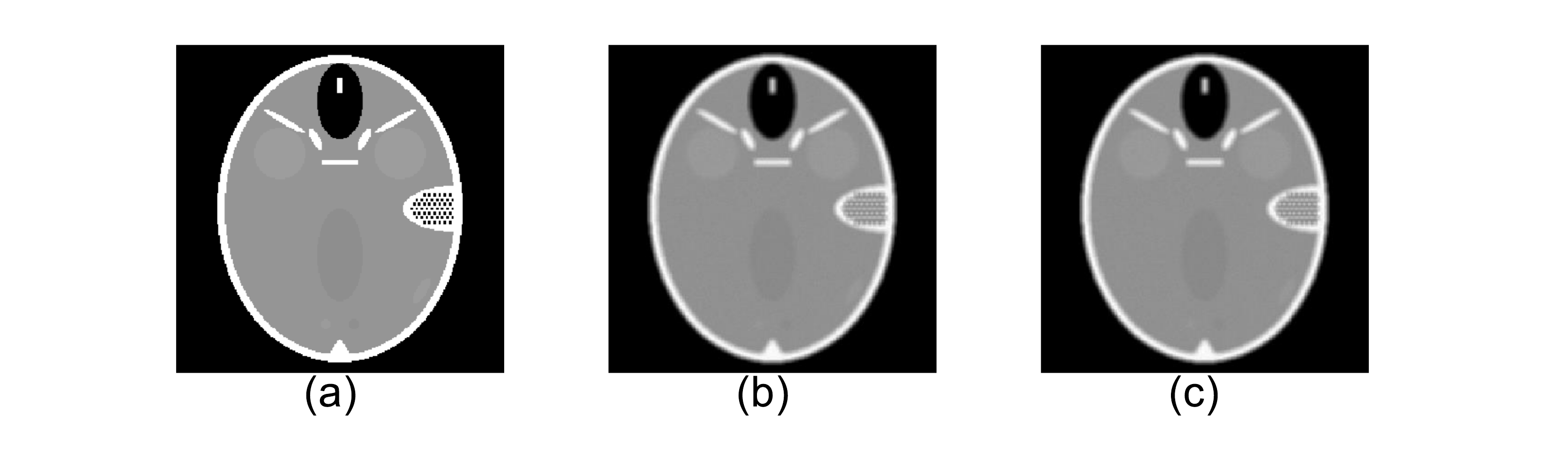}
	\caption{The original and noisy head phantom images. $(a)$ head phantom with 256×256 pixels; $(b)$ and $(c)$ blurred image with a 5$\times$5 uniform kernel and additive Gaussian noise with $\sigma=10^{ - 2}$ and $\sigma=10^{ - 3}$.}
	\label{FIG:1}
\end{figure*}

\begin{figure*}
	\subfigure[]{ 	
		\begin{minipage}[b]{0.5\textwidth} 
			\centering 
			\includegraphics[scale=.60]{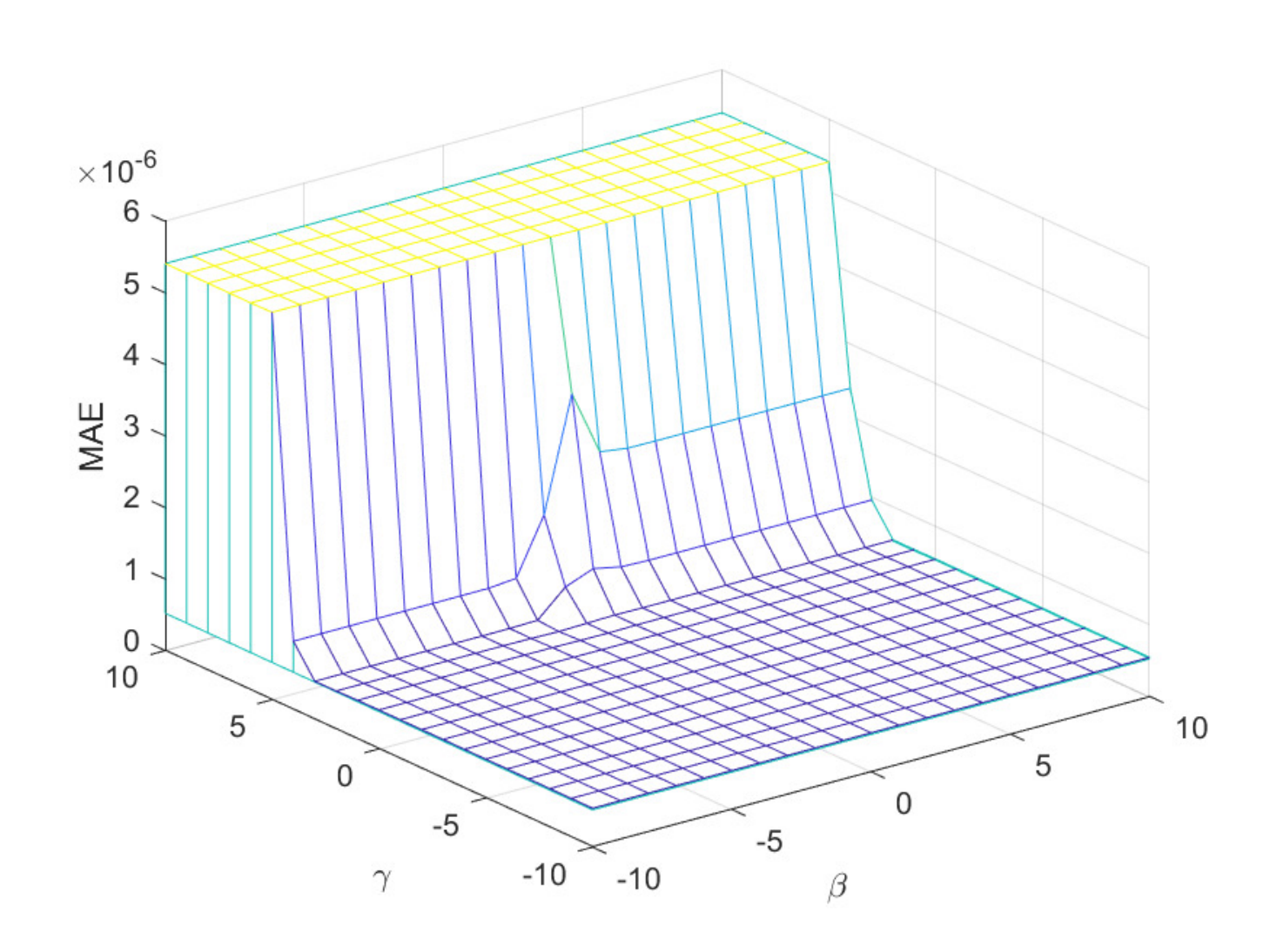} 
	\end{minipage}}%
	\subfigure[]{ 	
		\begin{minipage}[b]{0.5\textwidth} 
			\centering 
			\includegraphics[scale=.60]{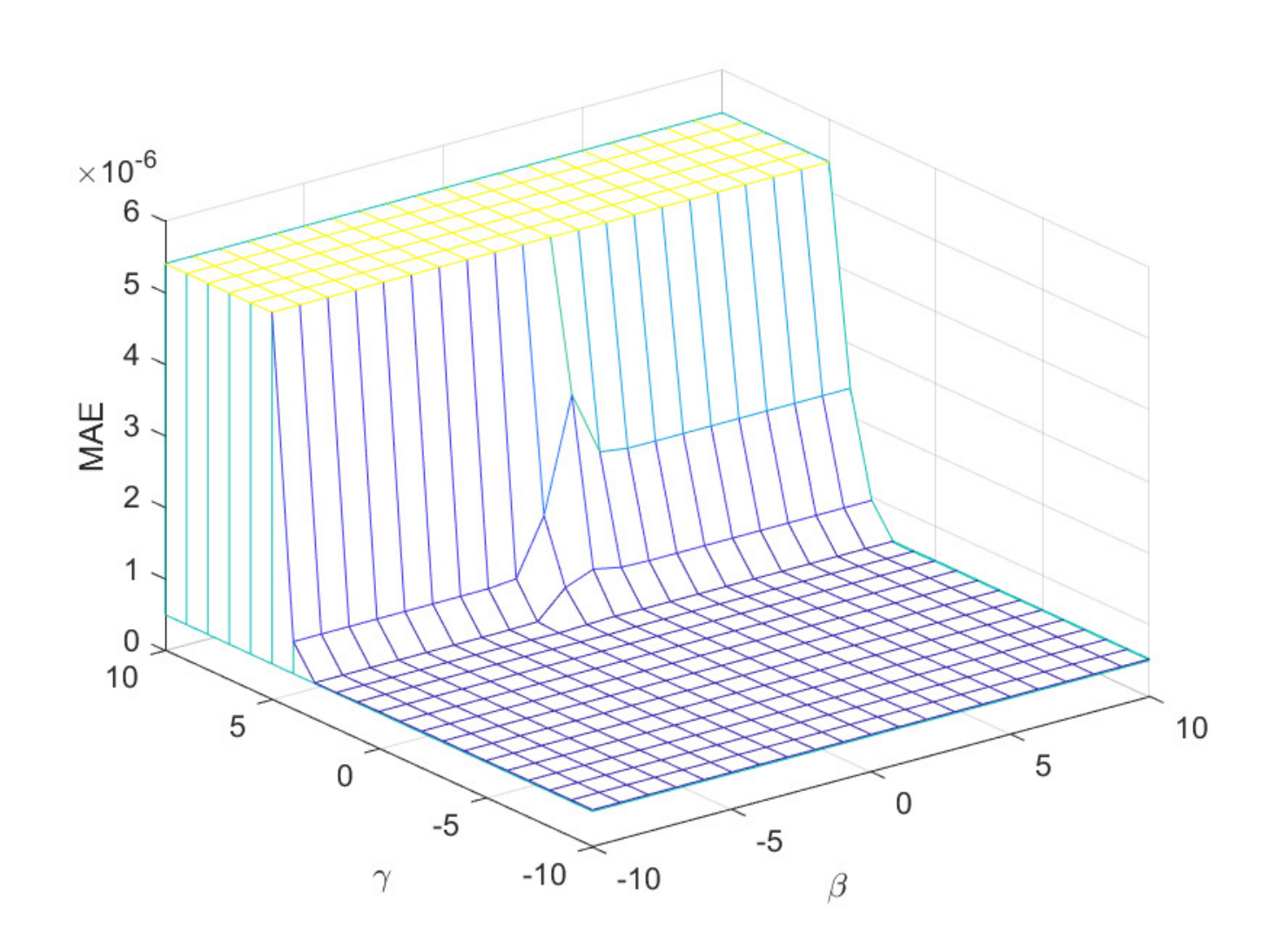} 
	\end{minipage}} 
	\caption{3D profile of $\beta$ and $\lambda$ on MAE with different Gaussian noise levels: (a) $\sigma=10^{-2}$ and (b) $\sigma=10^{-3}$.} 
	\label{FIG:hyperparameter} 
\end{figure*}

\begin{table}[width=.9\linewidth,cols=4,pos=h]
	\caption{The optimal MAE value and corresponding hyperparameter (Gaussian noise with $\sigma=10^{-2}$).}
	\label{tbl2}
	\begin{tabular*}{\tblwidth}{@{} LLLLL@{} }
		\toprule
		Termed &  $\beta$    & $\gamma$      & $\delta$      & MAE\\
		\midrule
		ISTA   & ${10^{-3}}$ & $-$           & $-$           & ${5.312077*10^{-7}}$ \\
		
		WLP    & ${10^{-5}}$ & ${10^{ - 10}}$& ${10^{ - 3}}$ & ${5.228672*10^{-7}}$ \\
		
		NW4    & ${10^{-5}}$ & ${10^{ - 2}}$ & ${10^{ - 3}}$ & ${5.410231*10^{-7}}$ \\
		 
		IRL1   & ${10^{-5}}$ & $-$           & ${10^{ - 3}}$ & ${5.228672*10^{-7}}$ \\
		
		ERIWSTA  & ${10^{2}}$& ${10^{-2}}$   & $-$           & ${5.218246*10^{-7}}$  \\
		\bottomrule
	\end{tabular*}
\end{table}

\begin{table}[width=.9\linewidth,cols=4,pos=h]
	\caption{The optimal MAE value and corresponding hyperparameter (Gaussian noise with $\sigma=10^{-3}$).}
	\label{tbl3}
	\begin{tabular*}{\tblwidth}{@{} LLLLL@{} }
		\toprule
		Termed &  $\beta$    & $\gamma$      & $\delta$      & MAE\\
		\midrule
		ISTA   & ${10^{-3}}$ & $-$           & $-$           & ${5.122013*10^{-7}}$ \\
		
		WLP    & ${10^{-5}}$ & ${10^{ - 5}}$ & ${10^{ - 3}}$ & ${5.018339*10^{-7}}$ \\
		
		NW4    & ${10^{-5}}$ & ${10^{ - 2}}$ & ${10^{ - 3}}$ & ${5.410231*10^{-7}}$ \\
		
		IRL1   & ${10^{-5}}$ & $-$           & ${10^{ - 3}}$ & ${5.018340*10^{-7}}$ \\
		
		ERIWSTA  & ${10^{2}}$& ${10^{-2}}$   & $-$           & ${5.005524*10^{-7}}$  \\	
		\bottomrule
	\end{tabular*}
\end{table}

\begin{figure*}
	\centering
	\subfigure[]{
		\includegraphics[scale=.5]{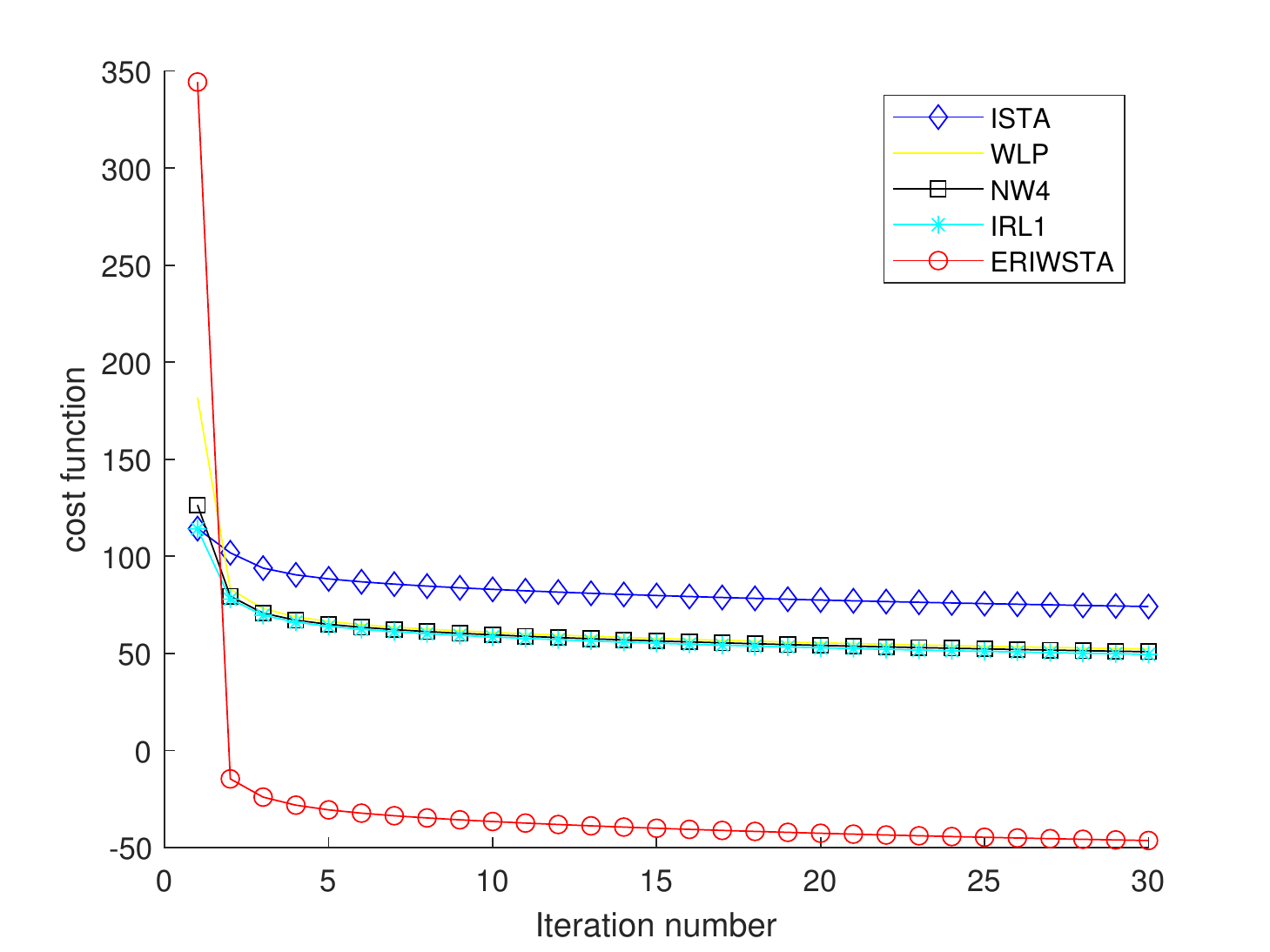}
	}
	\subfigure[]{
		\includegraphics[scale=.5]{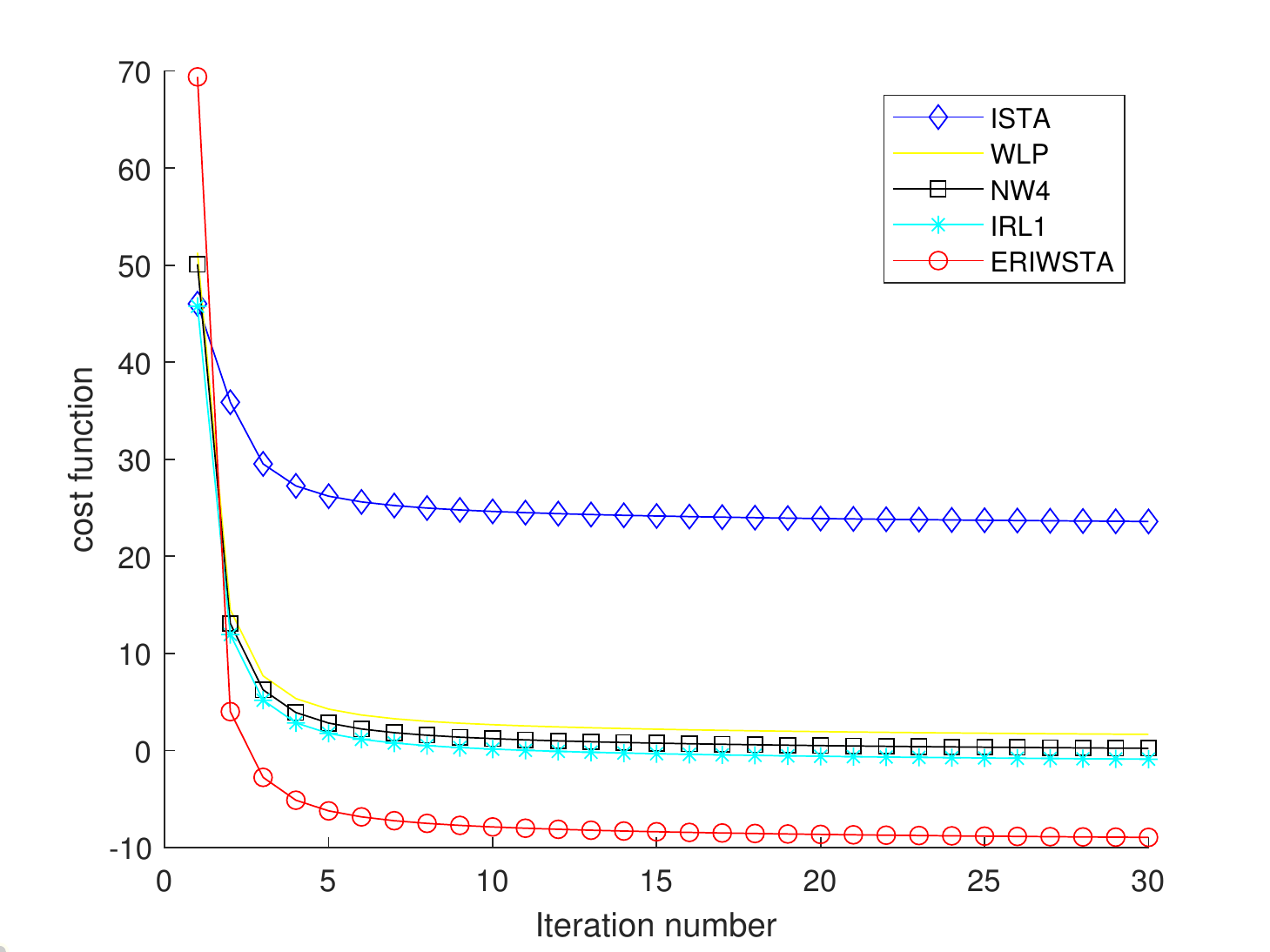} 
	}
	
	\caption{Cost function versus iteration number for different Gaussian noise levels: (a) $\sigma=10^{-2}$ and (b) $\sigma=10^{-3}$.}
	\label{FIG:CostCurve}
\end{figure*}

\begin{figure*}
	\centering
	\subfigure[]{
		\includegraphics[scale=.5]{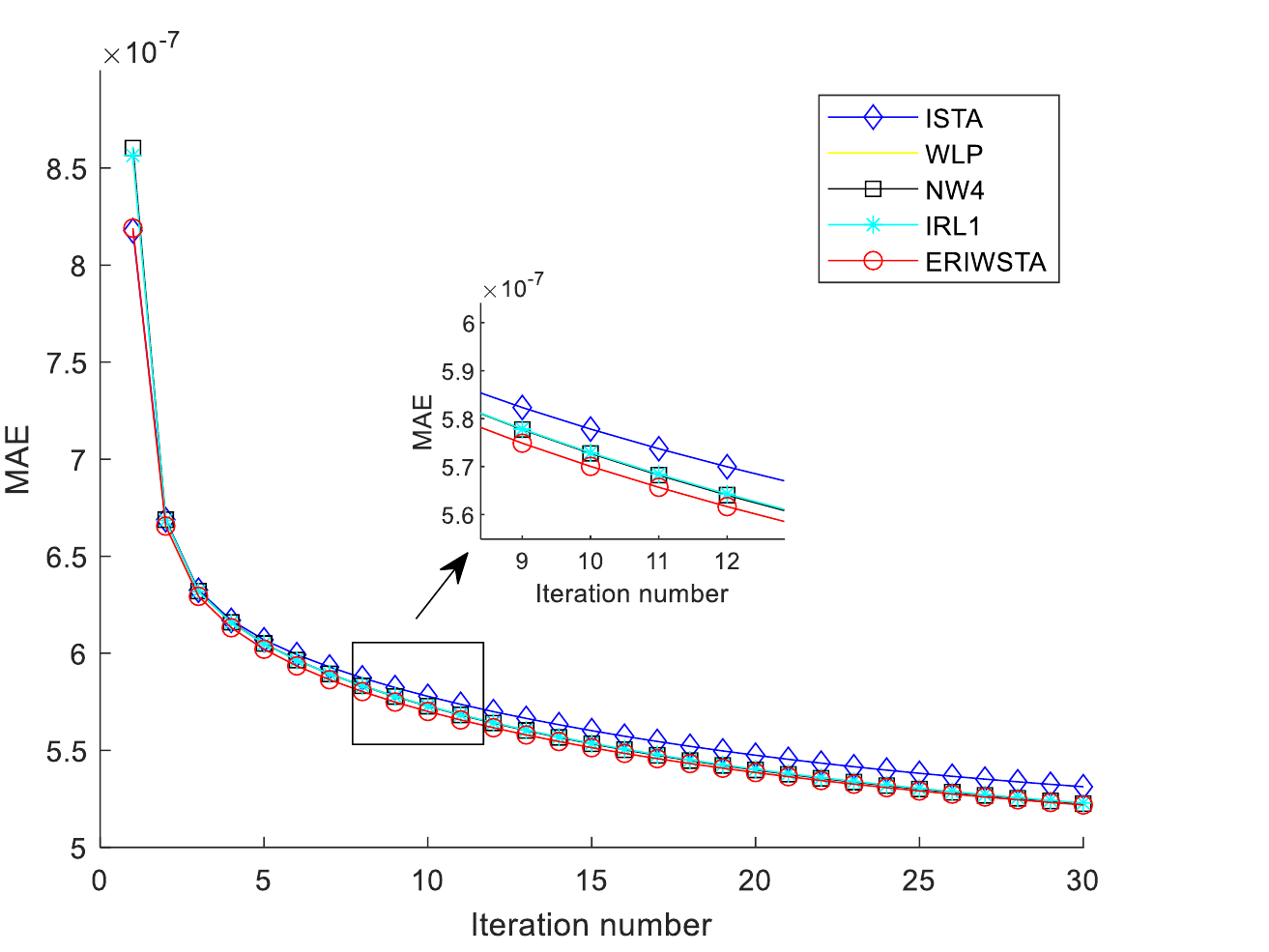}
	}
	\subfigure[]{
		\includegraphics[scale=.5]{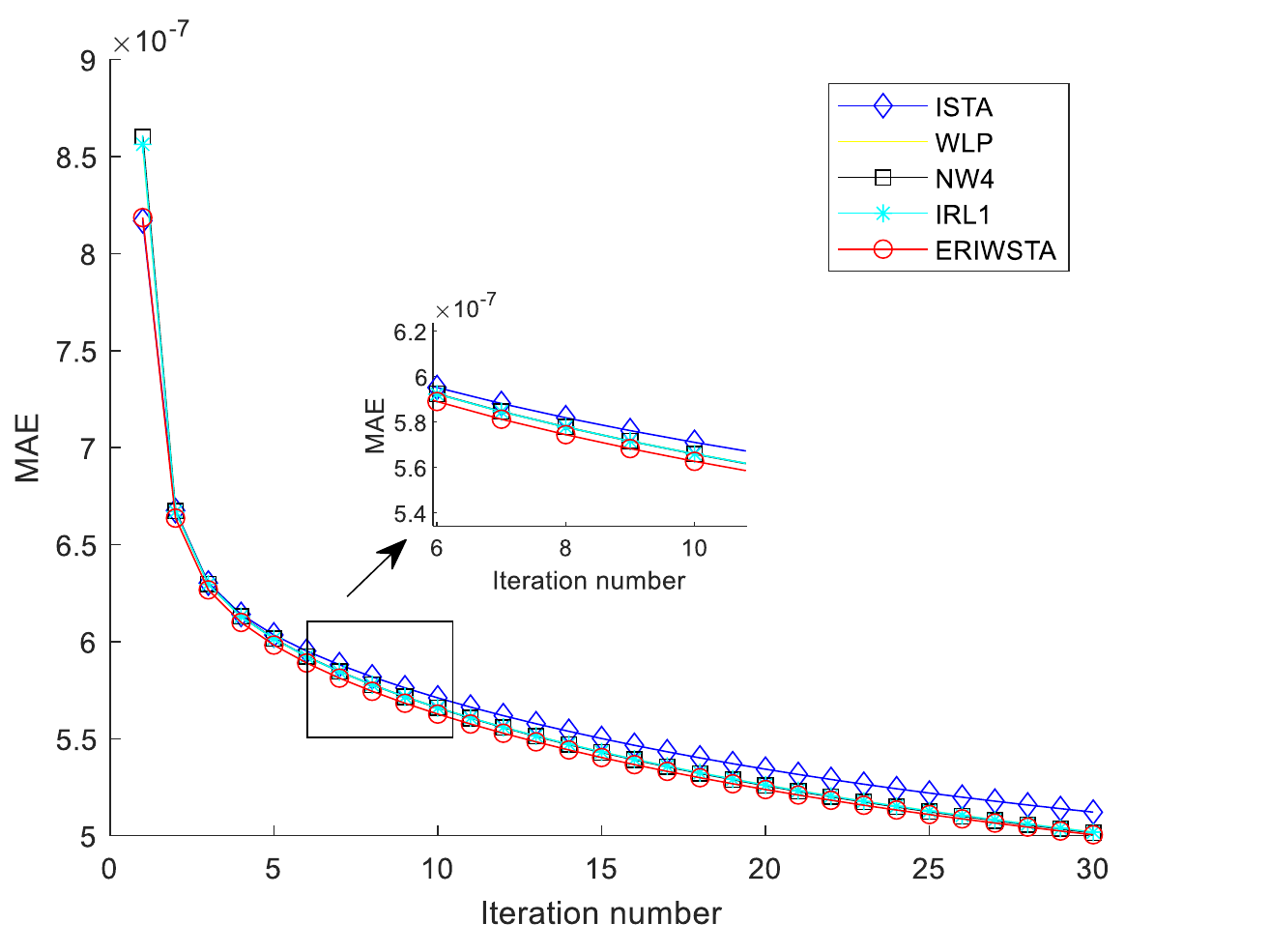} 
	}
	
	\caption{MAE versus iteration number for different Gaussian noise levels: (a) $\sigma=10^{-2}$ and (b) $\sigma=10^{-3}$.}
	\label{FIG:MAECurve}
\end{figure*}

\begin{figure}	 
	\centering 
	\includegraphics[scale=.75]{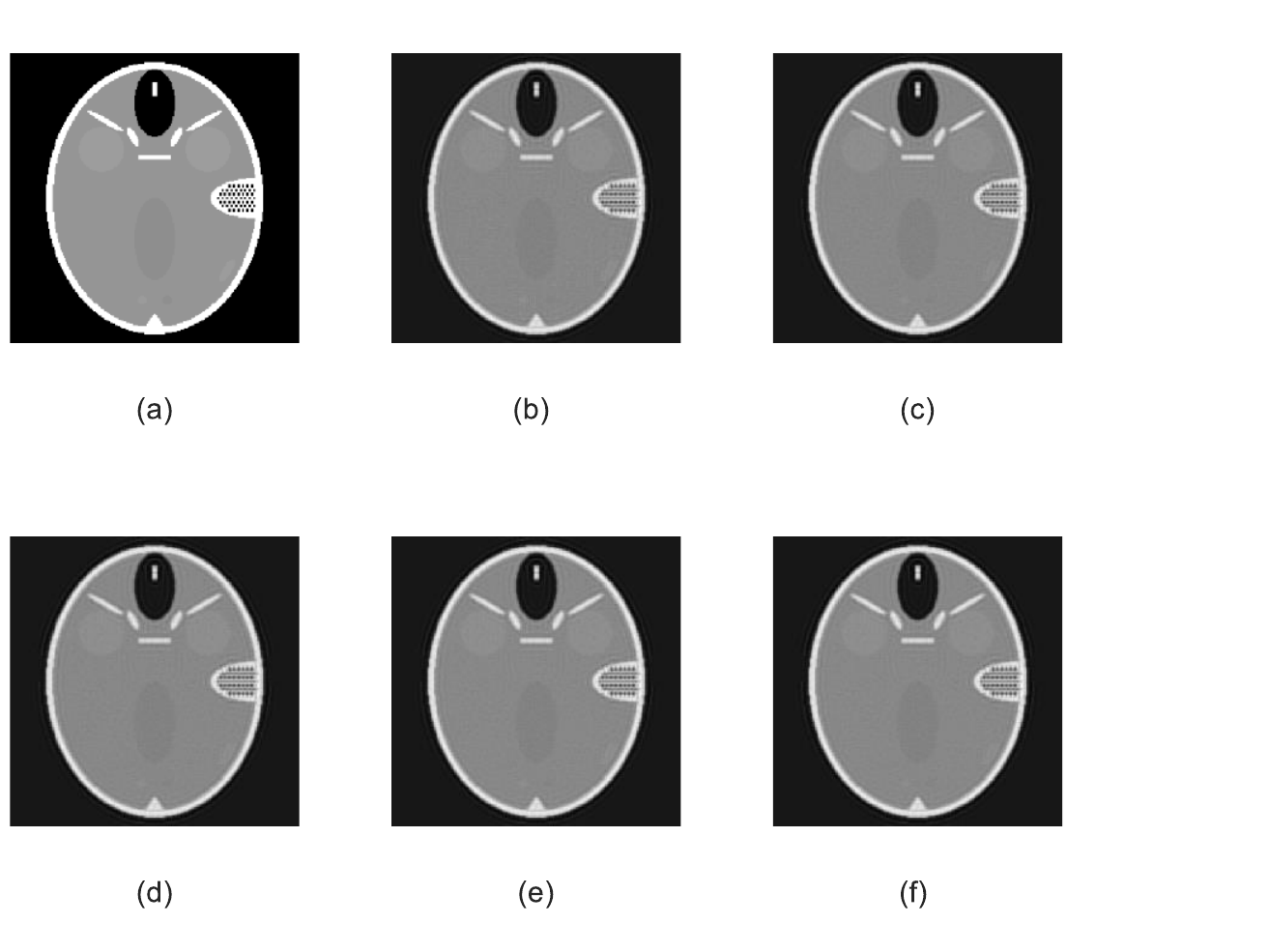} 
	\caption{After 30 iterations, the denoising results of ISTA, WLP, NW1, IRL1 and ERIWSTA with Gaussian noise with $\sigma=10^{-2}$.} 
	\label{FIG:ImageHighNoise}
\end{figure}

\begin{figure}
	\centering 
	\includegraphics[scale=.75]{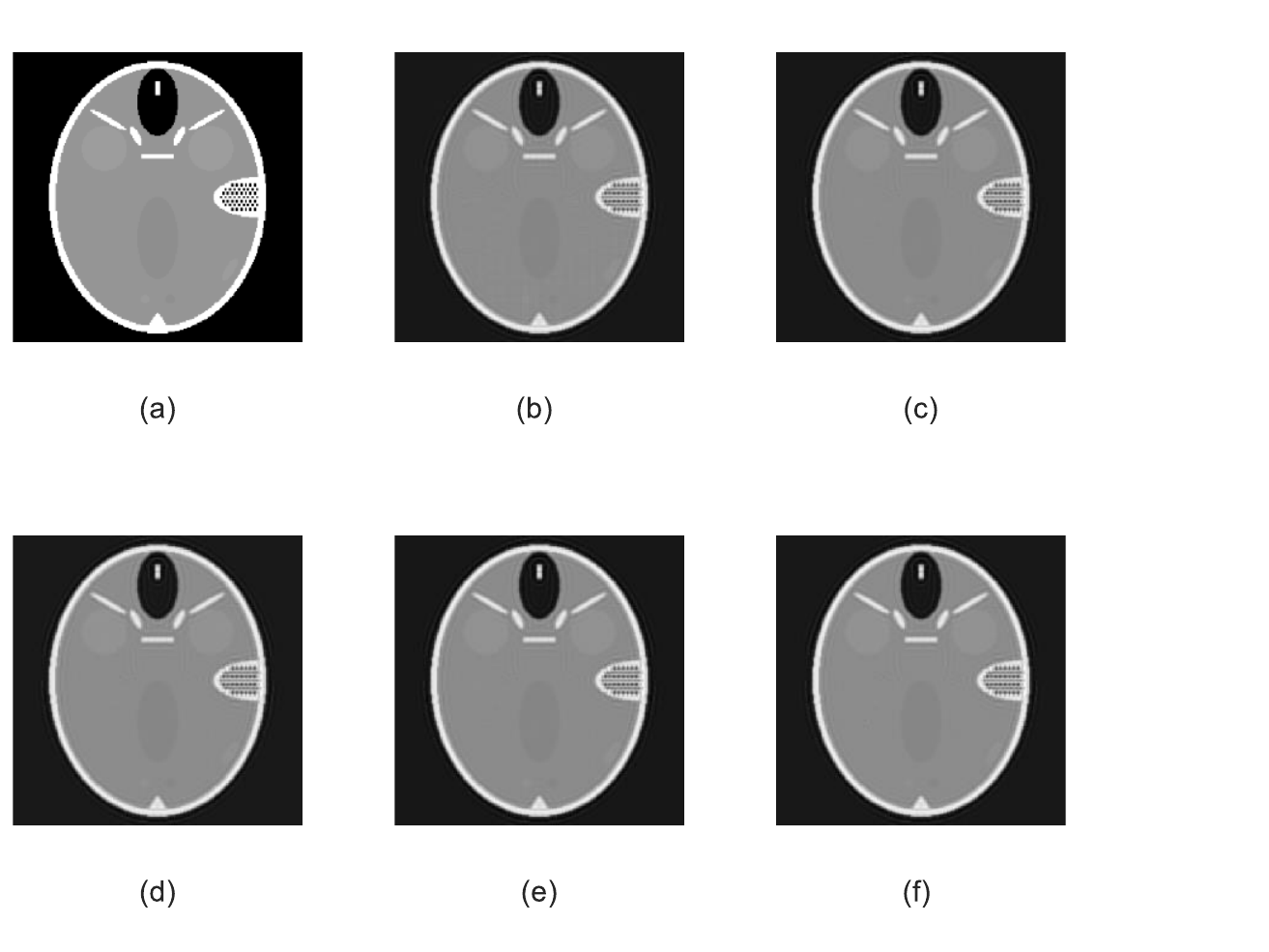} 
	\caption{After 30 iterations, the denoising results of ISTA, WLP, NW1, IRL1 and ERIWSTA with Gaussian noise with $\sigma=10^{-3}$.} 
	\label{FIG:ImageLowNoise}
\end{figure}

\begin{figure*}
	\centering
	\subfigure[]{
		\includegraphics[scale=.33]{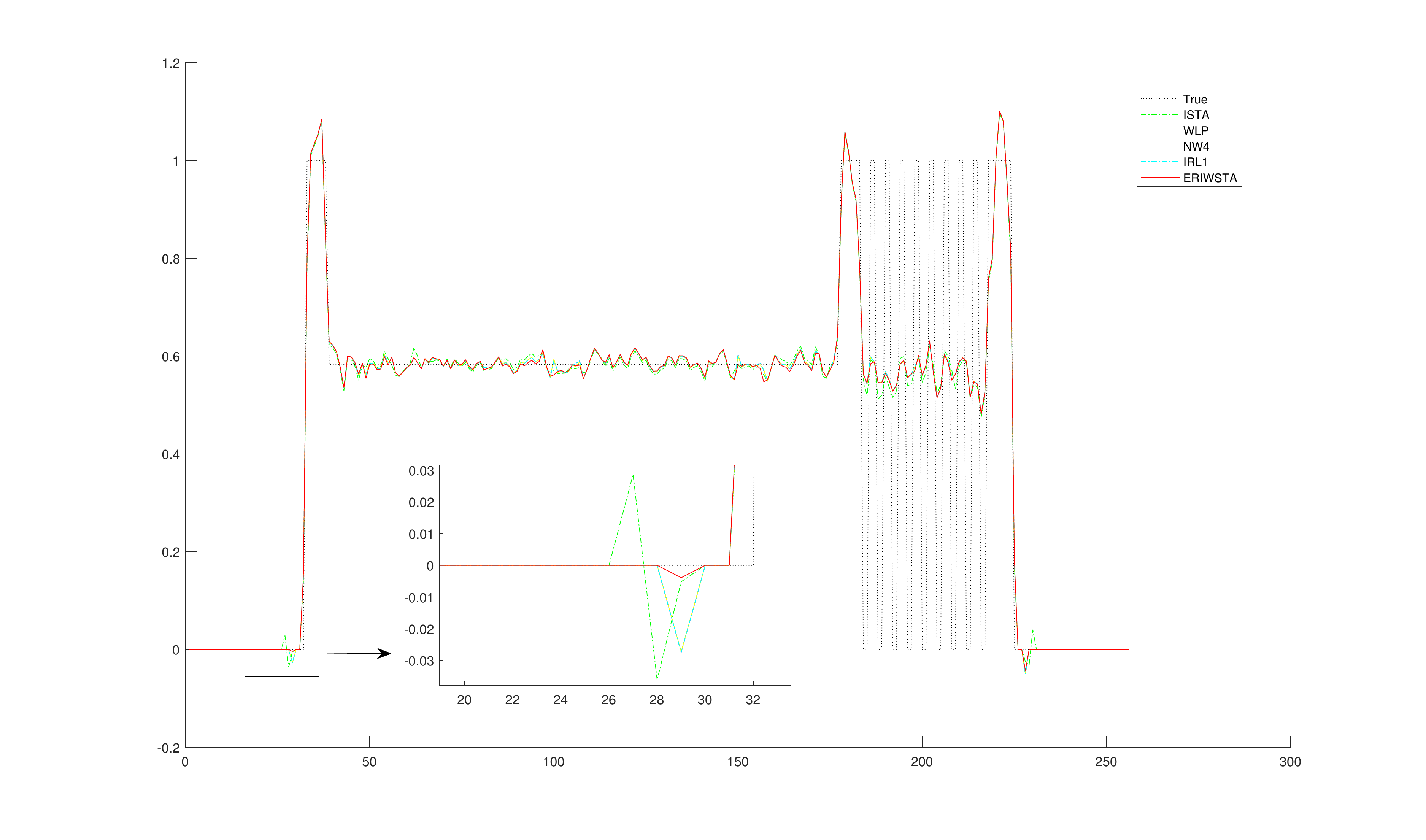}
	}
	\subfigure[]{
		\includegraphics[scale=.33]{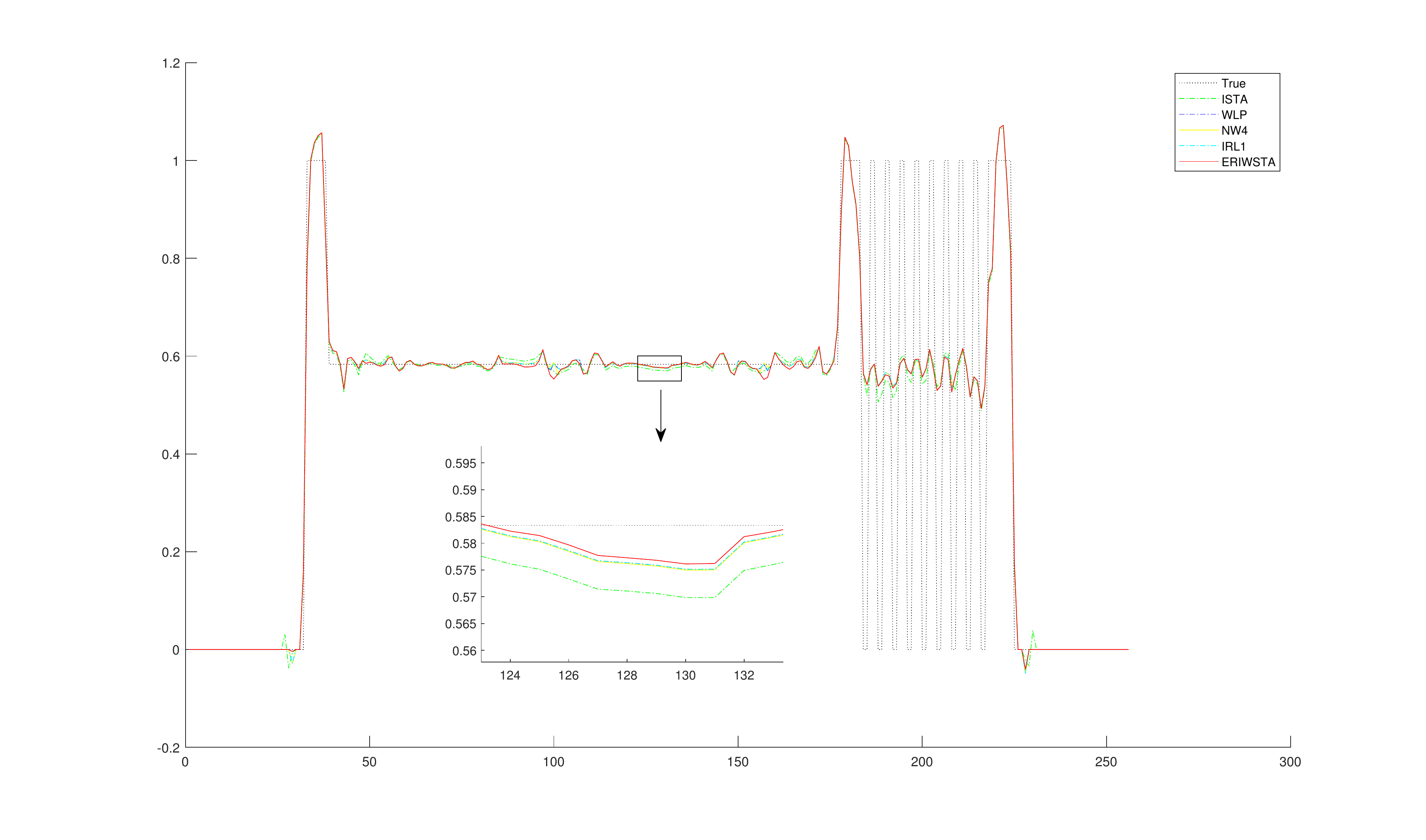} 
	}
	
	\caption{Horizontal central profiles of the restored images with different Gaussian noise levels: (a) $\sigma=10^{-2}$ and (b) $\sigma=10^{-3}$..}
	\label{fig:centralline1}
\end{figure*}

\begin{figure*}
	\centering
	\subfigure[]{
		\includegraphics[scale=.33]{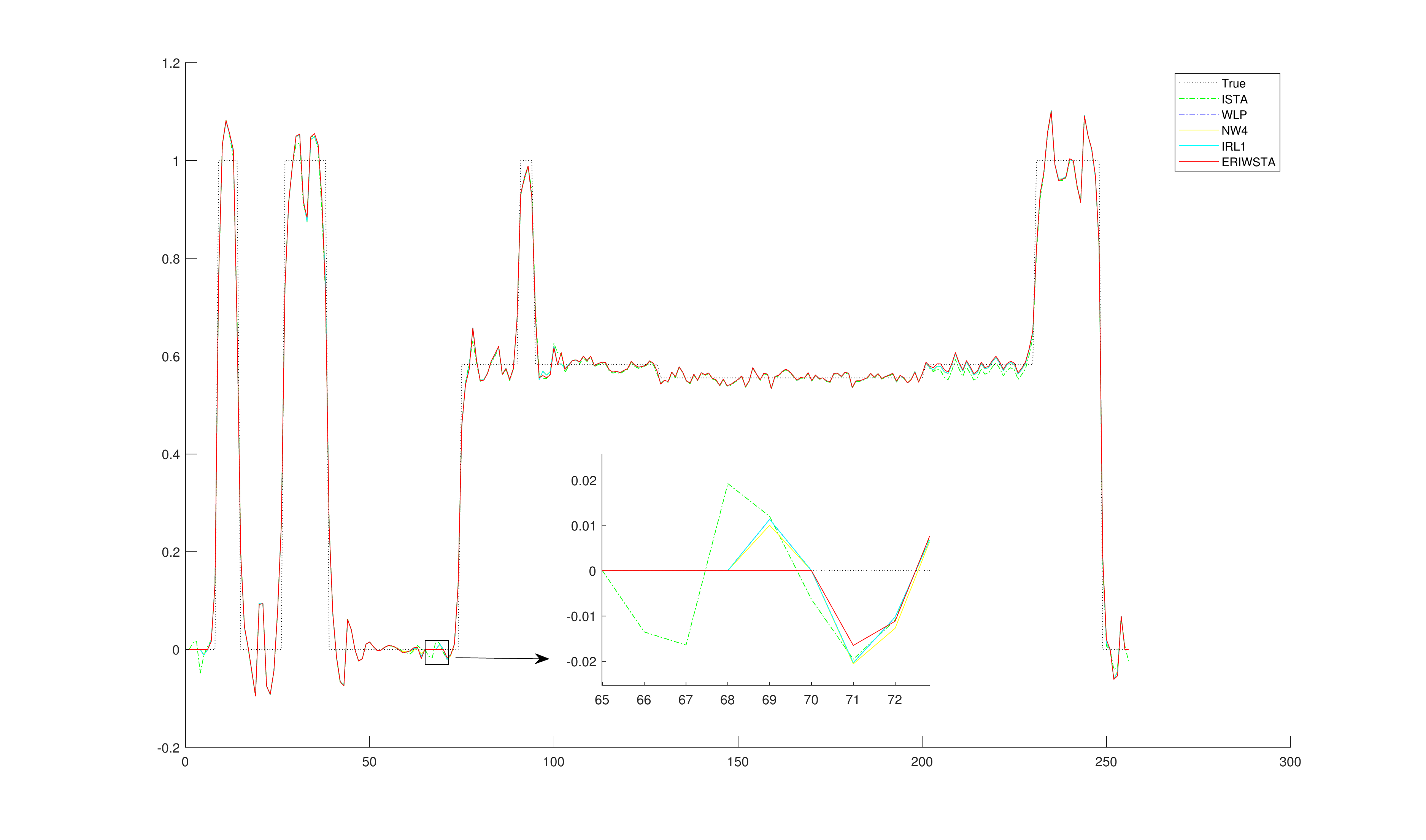}
	}
	\subfigure[]{
		\includegraphics[scale=.33]{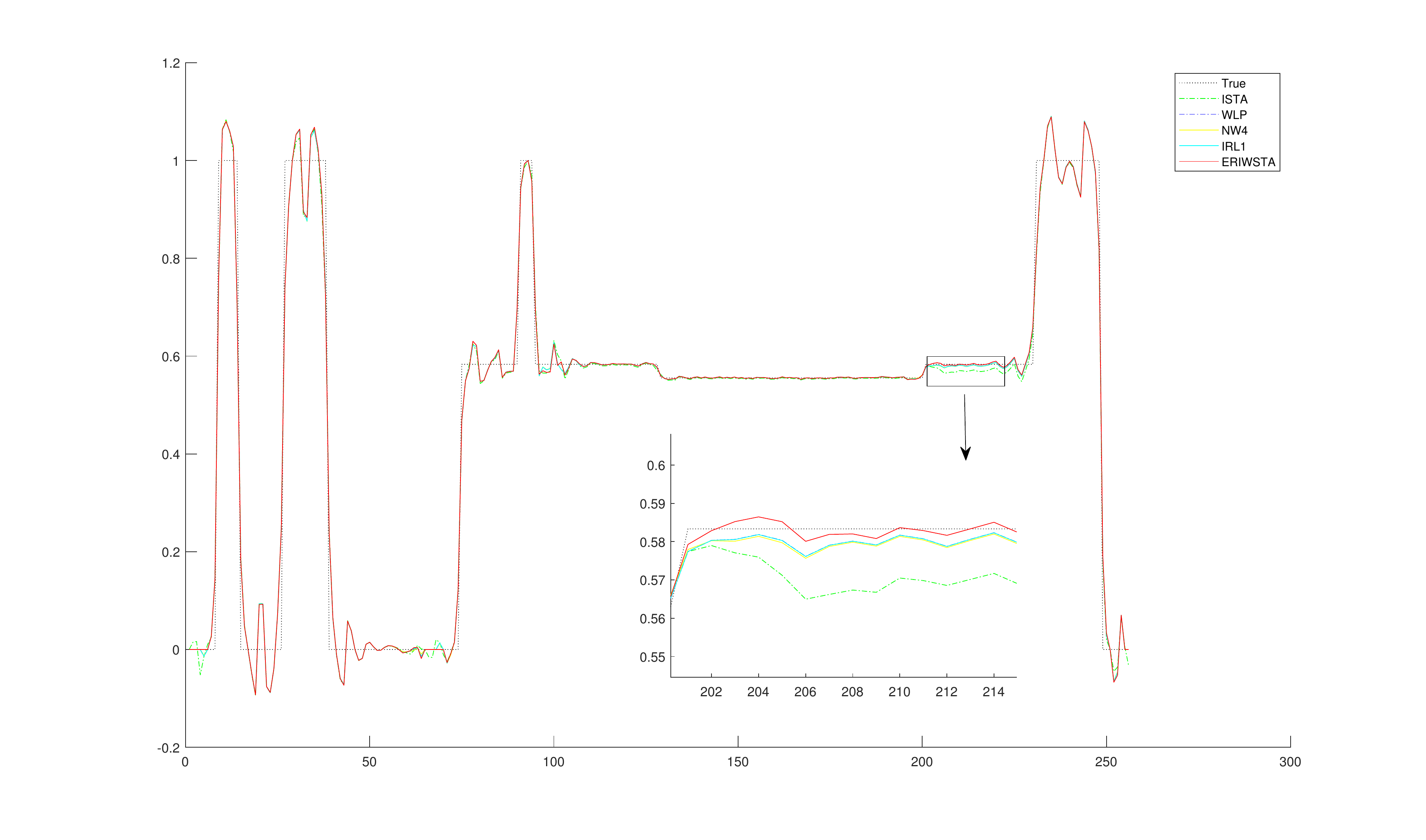} 
	}
	
	\caption{Vertical central profiles of the restored images with different Gaussian noise levels: (a) $\sigma=10^{-2}$ and (b) $\sigma=10^{-3}$.}
\label{fig:centralline2}
\end{figure*}

\section{Numerical experiments}
Numerical experiments are provided to evaluate the performance of the proposed ERWISTA compared with ISTA, WLP, NW4 and IRL1 on the denoising problem of computed tomography (CT) images. All experiments are performed on an HP computer with a 2.5 GHz Intel(R) Core(TM) i7-4710MQ CPU with 12 GB of memory using MATLAB R2019a for coding. A simulated Shepp-Logan phantom with $256\times256$ pixels was used to evaluate the algorithm performance, which is usually used in CT image analysis. There are many advantages to using simulated phantoms, including prior knowledge of the pixel values and the ability to control noise. We blurred the image by using a uniform $5\times5$ kernel (applied by the MATLAB function "\emph{fspecial}") and then added Gaussian noise by the following formula. We select $\sigma=10^{ - 2}$ and ${10^{ - 3}}$ as examples of high and low noise levels for the following experiments.
\begin{equation}
	x^{noise}=x^{true}+N(0,\sigma)
\end{equation}

Fig. \ref{FIG:1} shows the original and blurred-and-noisy images. Based on the good time-frequency localization characteristics of the wavelet transform, it can effectively distinguish high-frequency noise from low-frequency information. Therefore, the wavelet transform is used to reduce noise. The introduction of the wavelet matrix can also ensure the sparsity of the whole optimization algorithm. Without losing generality, let $A = PW$, where $P$ is the predetermined system matrix indicating the blurring information and $W$ represents the second-order Haar wavelet matrix.

Mean absolute error (MAE) was used to measure the similarity to the true image. The value of the MAE was calculated by taking the average of the squared differences between the restored pixel values and the true pixel values.
\begin{equation} 	
	MAE = \frac{1}{N}{||x^{restoration} - x^{true} ||_1}
\end{equation}

\subsection{Hyperparameter selection}

To select the penalty hyperparameter $\beta$ and the entropy weighted hyperparameter $\gamma$, we compare the MAE value after 100 iterations with respect to them from $10^{ - 10}$ to $10^{ 10}$. The results are shown
in Fig. \ref{FIG:hyperparameter}, demonstrating that ERIWSTA can achieve a consistently low MAE value over a wide range of $\beta$ and $\gamma$, which displays its robustness.

We also quantitatively display the optimal MAE and corresponding hyperparameters of the algorithms in Tabs. \ref{tbl2} and \ref{tbl3}. An interesting observation is that, regardless of low or high noise levels, the restoration accuracy of our algorithm is always better than the others. These optimal hyperparameters are also used in the following experiments.

\subsection{Algorithmic performance}
Fig. \ref{FIG:CostCurve} displays the cost function of the algorithms. As can be seen, the proposed algorithm always have the fast convergence speed, which arrive at the stable status early.

Fig. \ref{FIG:MAECurve} shows the MAE curves of the algorithms with respect to the number of iterations. The proposed ERIWSTA always has superior performance to the other algorithms, rapidly obtaining the minimum MAE value.

Figs. \ref{FIG:ImageHighNoise} and \ref{FIG:ImageLowNoise} indicate the denoising results with the given noise level. As can be seen, all of the algorithms achieve a similar image. Howver, Figs. \ref{fig:centralline1} and \ref{fig:centralline2}  quantitatively compares the vertical profiles of the
restored images with that of the true phantom in the central row and column.
We can see that ERIWSTA 
follows the outline of the phantom more accurately than the other algorithms.

\section{Conclusions}
In this paper, a new IWSTA type, called ERIWSTA, is proposed to solve the linear inverse problem. An entropy weighted term is introduced to measure the certainty of the weights, and then the Lagrange multiplier method is used to obtain a simple solution. The experimental results on image restoration of a synthetic CT head phantom show that ERIWSTA can achieve convergence faster with fewer iterations and better restoration accuracy than the other algorithms.

However, as with many existing algorithms, our algorithm also involves two main hyperparameters ($\beta$ and $\gamma$). In the future, we will focus on designing an automatic method to adjust these hyperparameters.

\section*{Acknowledgments}

This work was supported by the Fundamental Research Funds for the Central Universities (N2019006 and N180719020).

\printcredits

\end{document}